\documentclass[a4paper,11pt]{article}
\usepackage[margin=1.0in]{geometry}

\usepackage{amsmath}
\usepackage{algorithmic}
\usepackage{algorithm}
\usepackage{amsthm}
\usepackage{amsfonts}
\usepackage{comment}
\usepackage{graphicx}
\usepackage{hyperref}
\usepackage{color}

\newtheorem{theorem} {Theorem}
\newtheorem{lemma} {Lemma}
\newtheorem{definition} {Definition}

\newcommand{\alpharegret}{\alpha-\textrm{regret}}

\def\q{{\mathbf{q}}}
\def\x{{\mathbf{x}}}
\def\f{{\mathbf{f}}}
\def\s{{\mathbf{s}}}
\def\c{{\mathbf{c}}}
\def\h{{\mathbf{h}}}
\def\u{{\mathbf{u}}}
\def\v{{\mathbf{v}}}
\def\z{{\mathbf{z}}}
\def\w{{\mathbf{w}}}
\def\y{{\mathbf{y}}}
\def\p{{\mathbf{p}}}

\def\Q{{\mathbf{Q}}}

\def\E{{\mathbb{E}}}

\newcommand{\reals}{\mathbb{R}}

\newcommand{\mK}{\mathcal{K}}
\newcommand{\mF}{\mathcal{F}}
\newcommand{\mS}{\mathcal{S}}
\newcommand{\ball}{\mathcal{B}}

\newcommand{\dist}{\textrm{dist}}

\newcommand{\BS}{$\beta$-\textrm{BS}}

\newcommand{\convK}{\textrm{CH}(\mathcal{K})}
\newcommand{\conv}{\textrm{CH}}

\newcommand{\oraclek}{\mathcal{O}_{\mK}}
\newcommand{\oraclekext}{\hat{\mathcal{O}}_{\mK}}

\title{Efficient Online Linear Optimization \\with Approximation Algorithms}
\author{Dan Garber \\
Technion - Israel Institute of Technology \\
{\small{dangar@technion.ac.il}}}
\date{}

\begin{document}

 \maketitle

\begin{abstract}
We revisit the problem of \textit{online linear optimization} in case the set of feasible actions is accessible through an approximated linear optimization oracle with a factor $\alpha$ multiplicative approximation guarantee. This setting is in particular interesting since it captures natural online extensions of well-studied \textit{offline} linear optimization problems which are NP-hard, yet admit efficient approximation algorithms. The goal here is to minimize the $\alpha$\textit{-regret} which is the natural extension of the standard \textit{regret} in \textit{online learning} to this setting.
  We present new  algorithms with significantly improved oracle complexity for both the full information and bandit variants of the problem. Mainly, for both variants, we present $\alpha$-regret bounds of $O(T^{-1/3})$, were $T$ is the number of prediction rounds, using only $O(\log{T})$ calls to the approximation oracle per iteration, on average. These are the first results to obtain both average oracle complexity of $O(\log{T})$ (or even poly-logarithmic in $T$) and $\alpha$-regret bound $O(T^{-c})$ for a constant $c>0$, for both variants. 
\end{abstract}

\section{Introduction}
In this paper we revisit the problem of \textit{Online Linear Optimization} (OLO)  \cite{KV05}, which is a specialized case of \textit{Online Convex Optimization} (OCO) \cite{Hazan16} with linear loss functions, in case the feasible set of actions is accessible through an oracle for approximated linear optimization with a multiplicative approximation error guarantee. In the standard setting of OLO, a decision maker is repeatedly required to choose an action, a vector in some fixed feasible set in $\reals^d$. After choosing his action, the decision maker incurs loss (or payoff) given by the inner product between his selected vector and a vector chosen by an adversary. This game between the decision maker and the adversary then repeats itself. In the \textit{full information} variant of the problem, after the decision maker receives his loss (payoff) on a certain round, he gets to observe the vector chosen by the adversary. In the \textit{bandit} version of the problem, the decision maker only observes his loss (payoff) and does not get to observe the adversary's vector. The standard goal of the decision maker in OLO is to minimize a quantity known as \textit{regret}, which measures the difference between the average loss of the decision maker on a game of $T$ consecutive rounds (where $T$ is fixed and known in advance), and the average loss of the best feasible action in hindsight (i.e., chosen with knowledge of all actions of the adversary throughout the $T$ rounds) (in case of payoffs this difference is reversed). The main concern when designing algorithms for choosing the actions of the decision maker, is guaranteeing that the regret goes to zero as the length of the game $T$ increases, as fast as possible (i.e., the rate of the regret in terms of $T$). It should be noted that in this paper we focus on the case in which the adversary is \textit{oblivious} (a.k.a. \textit{non-adaptive}), which means the adversary chooses his entire sequence of actions for the $T$ rounds beforehand. 

While there exist well known algorithms for choosing the decision maker's actions which guarantee optimal regret bounds in $T$, such as the celebrated \textit{Follow the Perturbed Leader} (FPL) and \textit{Online Gradient Descent} (OGD) algorithms \cite{KV05, Zinkevich03, Hazan16}, efficient implementation of these algorithms hinges on the ability to efficiently solve certain convex optimization problems (e.g., linear minimization for  FPL or Euclidean projection for OGD) over the feasible set (or the convex hull of feasible points). However, when the feasible set corresponds for instance to the set of all possible solutions to some NP-Hard optimization problem, no such efficient implementations are known (or even widely believed to exist), and thus these celebrated regret-minimizing procedures cannot be efficiently applied. 
Luckily, many NP-Hard linear optimization problems (i.e., the objective function to either minimize or maximize is linear) admit efficient approximation algorithms with a multiplicative approximation guarantee. Some examples include \textsc{MAX-CUT} (factor $0.87856$ approximation due to \cite{Goemans95}) , \textsc{Metric TSP} (factor $1.5$ approximation due to \cite{Christofides76}), \textsc{Minimum Weighted Vertex Cover} (factor $2$ approximation \cite{Bar81}), and \textsc{Weighted Set Cover} (factor $(\log{}n+1)$ approximation due to \cite{Chvatal79}). It is thus natural to ask wether an efficient factor $\alpha$ approximation algorithm for an NP-Hard \textit{offline} linear optimization problem could be used to construct, in a generic way, an efficient algorithm for the \textit{online} version of the problem. Note that in this case, even efficiently computing the best fixed action in hindsight is not possible, and thus, minimizing regret via an efficient algorithm does not seem likely (given an approximation algorithm we can however compute in hindsight a decision that corresponds to at most (at least) $\alpha$ times the average loss (payoff) of the best fixed decision in hindsight). 

In their paper \cite{KKL}, Kakade, Kalai and Ligett were the first to address this question in a fully generic way. They showed that using only an $\alpha$-approximation oracle for the set of feasible actions, it is possible, at a high level, to construct an online algorithm which achieves vanishing (expected) $\alpha$-regret, which is the difference between the average loss of the decision maker and $\alpha$ times the average loss of the best fixed point in hindsight (for loss minimization problems and $\alpha \geq 1$; a corresponding definition exists for payoff maximization problems and $\alpha < 1$). Concretely, \cite{KKL} showed that one can guarantee $O(T^{-1/2})$ expected $\alpha$-regret in the full-information setting, which is optimal, and $O(T^{-1/3})$ in the bandit setting under the additional assumption of the availability of a \textit{Barycentric Spanner} (which we discuss in the sequel).

While the algorithm in \cite{KKL} achieves an optimal $\alpha$-regret bound (in terms of $T$) for the full information setting, in terms of computational complexity, the algorithm requires, in worst case, to perform on each round $O(T)$ calls to the approximation oracle, which might be prohibitive and render the algorithm inefficient, since as discussed, in general, $T$ is assumed to grow to infinity and thus the dependence of the runtime on $T$ is of primary interest. Similarly, their algorithm for the bandit setting requires $O(T^{2/3})$ calls to the approximation oracle per iteration.

The main contribution of our work is in providing new low $\alpha$-regret algorithms for the full information and bandit settings with significantly improved oracle complexities. A detailed comparison with \cite{KKL} is given in Table \ref{table:results}.
Concretely, for the full-information setting, we show it is possible to achieve $O(T^{-1/3})$ expected $\alpha$-regret using only $O(\log(T))$ calls to the approximation oracle per iteration, on average, which significantly improves over the $O(T)$ bound of \cite{KKL}\footnote{as we show in the sequel, even if we relax the algorithm of \cite{KKL} to only guarantee $O(T^{-1/3})$ $\alpha$-regret, it will still require $O(T^{2/3})$ calls to the oracle per iteration, on average.}.
We also show a bound of $O(T^{-1/2})$ on the expected $\alpha$-regret (which is optimal) using only $O(\sqrt{T}\log(T))$ calls to the oracle per iteration, on average, which gives nearly quadratic improvement over \cite{KKL}. In the bandit setting we show it is possible to obtain a $O(T^{-1/3})$ bound on the expected $\alpha$-regret (same as in \cite{KKL}) using only $O(\log(T))$ calls to the oracle per iteration, on average, under the same assumption on the availability of a \textit{Barycentric Spanner} (BS). It is important to note that while there exist algorithms for OLO with bandit feedback which guarantee $\tilde{O}(T^{-1/2})$ expected regret \cite{Abernethy08, Hazan14vol} (where the $\tilde{O}(\cdot)$ hides poly-logarithmic factors in $T$), these require on each iteration to either solve to arbitrarily small accuracy a convex optimization problem over the feasible set \cite{Abernethy08}, or sample a point from the feasible set according to a specified distribution 
\cite{ Hazan14vol}, both of which cannot be implemented efficiently in our setting. On the other-hand, as we formally show in the sequel, at a high level, using a BS (originally introduced in \cite{Awerbuch04}) simply requires to find a single set of $d$ points from the feasible set which span the entire space $\reals^d$ (assuming this is possible, otherwise the set could be mapped to a lower dimensional space) and store them in memory. The process of finding these vectors can be viewed as a preprocessing step and thus can be carried out offline. Moreover, as discussed in \cite{KKL}, for many NP-Hard problems it is possible to compute a BS in polynomial time and thus even this preprocessing step is efficient. Importantly, \cite{KKL} shows that the approximation oracle by itself is not strong enough to guarantee non-trivial $\alpha$-regret in the bandit setting, and hence this assumption on the availability of a BS seems reasonable. Since the best general regret bound known using a  BS is $O(T^{-1/3})$, the $\alpha$-regret bound of our bandit algorithm is the best achievable to date via an efficient algorithm.

Technically, the main challenge in the considered setting is that as discussed, we cannot optimize over the feasible set (or its convex hull) and thus cannot readily apply standard tools such as FPL and OGD. In \cite{KKL} it was shown that despite this fact, it is possible to use the approximation oracle and the  OGD method to generate two sequences of points such that one sequence, which is the output of OGD, while being \textit{infeasible}, still achieves low $\alpha$-regret with respect to the feasible set of actions. The second sequence of points, \textit{is feasible} and dominates (point wise) the sequence of OGD for every relevant linear loss (payoff) function. Thus, playing the second feasible sequence guarantees low $\alpha$-regret. The projection step of OGD is replaced in their work with an iterative algorithm, which at a high-level is based on an approach very similar to the classical Frank-Wolfe algorithm for convex optimization, which finds an \textit{infeasible} point, but one that both satisfies the projection property required by OGD and is dominated by a convex combination of feasible points for every relevant linear loss (payoff) function. Unfortunately, as we show in the sequel, in worst case, the number of queries to the approximation oracle, required by this so-called projection algorithm per iteration, is linear in $T$ (or $T^{2/3}$ in the bandit setting), which does not seem improvable with this approach. In this work, while our online algorithms are also based on an application of OGD, our approach to computing the so-called projections is drastically different, and is based on a coupling of two \textit{cutting plane methods}, one that is based on the Ellipsoid method, and the other that resembles Gradient Descent. This approach might be of independent interest and might prove useful to other problems in which it is reasonable to assume that the feasible set is accessed through an approximated linear optimization oracle with a multiplicative approximation guarantee.

\begin{table*}[h!]\renewcommand{\arraystretch}{1.3}
{\footnotesize
\begin{center}
  \begin{tabular}{| l | c | c | c | c |}    \hline
      & \multicolumn{2}{|c|}{full information} &\multicolumn{2}{|c|}{bandit information}\\ \hline
    Reference &  $\alpharegret$ & oracle complexity &  $\alpharegret$ & oracle complexity \\ \hline 
    KKL \cite{KKL} & $ T^{-1/2}$  & $T$  & $ T^{-1/3}$  & $T^{2/3}$\\ \hline    
    This paper  (Thm. \ref{thm:newOGD}, \ref{thm:newBandit}) & $ T^{-1/3}$  & $\log(T)$ & $ T^{-1/3}$  & $\log(T)$ \\ \hline    
    This paper  (Thm. \ref{thm:newOGD}) & $ T^{-1/2}$ & $\sqrt{T}\log(T)$ & - & - \\ \hline
  \end{tabular}
  \caption{comparison of expected $\alpharegret$ bounds and average number of calls to the approximation oracle per iteration. In all bounds we give only the dependence on the length of the game $T$ and omit all other dependencies which we treat as constants. In the bandit setting we report the \textit{expected} number of calls to the oracle per iteration.
  }\label{table:results}
\end{center}
}
\vskip -0.2in
\end{table*}\renewcommand{\arraystretch}{1}


\subsection{Additional related work}
Kalai and Vempala \cite{KV05} showed that approximation algorithms which have \textit{point-wise approximation guarantee}, such as the celebrated \textsc{MAX-CUT} algorithm of \cite{Goemans95}, could be used to instantiate their \textit{Follow the Perturbed Leader} framework to achieve low $\alpha$-regret. However this construction is far from generic and requires the oracle to satisfy additional non-trivial conditions. This approach was also used in \cite{Balcan06}. In \cite{KV05} it was also shown that FPL could be instantiated with a FPTAS to achieve low $\alpha$-regret, however the approximation factor in the FPTAS needs to be set to roughly $(1+O(T^{-1/2}))$, which may result in prohibitive running times even if a FPTAS for the underlying problem is available.
Similarly, in \cite{Fujita} it was shown that if the approximation algorithm is based on solving a convex relaxation of the original, possibly NP-Hard, problem, this additional structure can be used with the FPL framework to achieve low $\alpha$-regret efficiently. To conclude all of the latter works consider specialized cases in which the approximation oracle satisfies additional non-trivial assumptions beyond its approximation guarantee, whereas here, similarly to \cite{KKL}, we will be interested in a generic as possible conversion from the offline problem to the online one, without imposing additional structure on the offline oracle.

\subsection{Organization of the paper}
The rest of this paper is organized as follows. 
In Section \ref{sec:prelim} we give a formal description of our setting, including the full-information and bandit variants, present basic algorithmic tools that were used in previous work and will be used in our online algorithms as-well, and discuss in more detail the previous work of \cite{KKL}. In Section \ref{sec:project} we detail our main technical contribution and the algorithmic basis to our efficient online algorithms - an oracle-efficient algorithm for computing (infeasible) projections onto a convex set using an approximation oracle. Finally, in Section \ref{sec:algorithms} we present our online algorithms for the full-information and bandit settings and give formal guarantees on their regret bounds and oracle complexities.

\section{Preliminaries}\label{sec:prelim}

\subsection{Online linear optimization with approximation oracles}
Let $\mK, \mF$ be compact sets of points in $\reals^d_+$ (non-negative orthant in $\reals^d$) such that $\max_{\x\in\mK}\Vert{\x}\Vert \leq R, \max_{\f\in\mF}\Vert{\f}\Vert \leq F$, for some $R > 0, F>0$ (throughout this work we let $\Vert\cdot\Vert$ denote the standard Euclidean norm), and for all $\x\in\mK, \f\in\mF$ it holds that $C \geq \x\cdot \f \geq 0$, for some $C >0$.

We assume $\mK$ is  accessible through an approximated linear optimization oracle $\oraclek:\reals^d_+\rightarrow\mK$ with parameter $\alpha > 0$ such that:
\begin{eqnarray*}
\forall \c\in\reals^d_+: \qquad \oraclek(\c)\in\mK \quad \textrm{and} \quad 
\left\{ \begin{array}{ll}
          \oraclek(\c)\cdot \c \leq \alpha\min_{\x\in\mK}\x\cdot \c & \mbox{if $\alpha \geq 1$} ;\\ 
         \oraclek(\c)\cdot \c \geq \alpha\max_{\x\in\mK}\x\cdot \c & \mbox{if $\alpha < 1$}.\end{array} \right. 
\end{eqnarray*}
Here $\mK$ is the feasible set of actions for the player, and $\mF$ is the set of all possible loss/payoff vectors. We note that both of our assumptions that $\mK\subset\reals^d_+, \mF\subset\reals^d_+$ and that the oracle takes inputs from $\reals^d_+$ are made for ease of presentation and clarity, and since these naturally hold for many NP-Hard optimization problem that are relevant to our setting. Nevertheless, these assumptions could be easily generalized as done in \cite{KKL}.

Since naturally a factor $\alpha > 1$ for the approximation oracle is reasonable only for loss minimization problems, and a value $\alpha < 1$ is reasonable for payoff maximization problems, throughout this work it will be convenient to use the value of $\alpha$ to differentiate between minimization problems and maximization problems.

Given a sequence of linear loss/payoff functions $\{\f_1,...,\f_T\}\in\mF^T$ and a sequence of feasible points $\{\x_1,....,\x_T\}\in\mK^T$, we define the $\alpharegret$ of the sequence $\{\x_t\}_{t\in[T]}$ with respect to the sequence $\{\f_t\}_{t\in[T]}$ as
\begin{eqnarray}\label{eq:alphaRegret}
\alpharegret(\{(\x_t,\f_t)\}_{t\in[T]}) := \left\{ \begin{array}{ll}
          \frac{1}{T}\sum_{t=1}^T\x_t \cdot \f_t - \alpha\cdot\min_{\x\in\mK}\frac{1}{T}\sum_{t=1}^T\x\cdot \f_t & \mbox{if $\alpha \geq 1$} ;\\ \\
          \alpha\cdot\max_{\x\in\mK}\frac{1}{T}\sum_{t=1}^T\x\cdot \f_t - \frac{1}{T}\sum_{t=1}^T\x_t \cdot \f_t& \mbox{if $\alpha < 1$}.\end{array} \right. 
\end{eqnarray}
When the sequences $\{\x_t\}_{t\in[T]},\{\f_t\}_{t\in[T]}$ are obvious from context we will simply write $\alpharegret$ without stating these sequences. Also, when the sequence $\{\x_t\}_{t\in[T]}$ is randomized we will use $\E[\alpharegret]$ to denote the expected $\alpha$-regret.

\subsubsection{Online linear optimization with full information}
In OLO with full information, we consider a repeated game of $T$ prediction rounds, for a fixed $T$, where on each round $t$, the decision maker is required to choose a feasible action $\x_t\in\mK$. After committing to his choice, a linear loss function $\f_t\in\mF$ is revealed, and the decision maker incurs loss of $\x_t\cdot\f_t$. In the payoff version, the decision maker incurs payoff of $\x_t\cdot\f_t$.
The game then continues to the next round.
The overall goal of the decision maker is to guarantee that $\alpharegret(\{(\x_t,\f_t)\}_{t\in[T]}) = O(T^{-c})$ for some $c >0$, at least in expectation (in fact using randomization is mandatory since $\mK$ need not be convex). 
Here we assume that the adversary is \textit{oblivious} (aka \textit{non-adaptive}), i.e., the sequence of losses/payoffs $\f_1,...,\f_T$ is chosen in advance (before the first round), and does not depend on the actions of the decision maker.

\subsubsection{Bandit feedback}
The bandit version of the problem is identical to the full information setting with one crucial difference: on each round $t$, after making his choice, the decision maker does not observe the vector $\f_t$, but only the value of his loss/payoff, given by $\x_t\cdot\f_t$. The goal is again to guarantee expected $\alpha$-regret that vanishes as $T$ grows to infinity.

\subsection{Additional notation}
For any two sets $\mS,\mK\subset\reals^d$ and a scalar $\beta\in\reals$ we define the sets
\begin{eqnarray*}
\mS+\mK &:=& \{\x +\y ~ | ~ \x\in\mS,~\y\in\mK\}, \qquad \beta\mS := \{\beta\x ~ | ~ \x\in\mS\}.
\end{eqnarray*}
We also denote by $\convK$ the convex-hull of all points in a set $\mK$.

For a convex and compact set $\mS\subset\reals^d$ and a point $\x\in\reals^d$ we define
\begin{eqnarray*}
\dist(\x,\mS) := \min_{\z\in\mS}\Vert{\z-\x}\Vert.
\end{eqnarray*} 
We let $\ball(\c,r)$ denote the Euclidean ball or radius $r$ centered in $\c$.

\subsection{Basic algorithmic tools and the KKL approach}
We now briefly describe two very basic ideas that are essential for constructing our algorithms, namely the \textit{extended approximation oracle} and the \textit{online gradient descent without feasibility} method. These were already suggested in \cite{KKL} to obtain their low $\alpha$-regret algorithms. We then briefly describe the approach of \cite{KKL} and discuss its shortcomings in obtaining oracle-efficient algorithms.

\subsubsection{The extended approximation oracle}\label{sec:extended_oracle}

As discussed, a key difficulty of our setting that prevents us from directly applying well studied algorithms for OLO, is that essentially all standard algorithms require to exactly solve (or up to arbitrarily small error) some linear/convex optimization problem over the convexification of the feasible set $\conv(\mK)$. However, not only that our approximation oracle $\oraclek(\cdot)$ cannot perform exact minimization, even for $\alpha = 1$ it is applicable only with inputs in $\reals^d_+$, and hence cannot optimize in all directions.

A natural approach, suggested in \cite{KKL}, to overcome the approximation error of the oracle $\oraclek(\cdot)$, is to consider optimization with respect to the convex set $\conv(\alpha\mK)$ (i.e. convex hull of all points in $\mK$ scaled by a factor of $\alpha$) instead of $\conv(\mK)$. Indeed, if we consider for instance the case $\alpha \geq 1$, it is straightforward to see that for any $\c\in\reals^d_+$,
\begin{eqnarray*}
\oraclek(\c)\cdot \c \leq \alpha\min_{\x\in\mK}\x\cdot\c = \alpha\min_{\x\in\conv(\mK)}\x\cdot\c = \min_{\x\in\conv(\alpha\mK)}\x\cdot\c.
\end{eqnarray*}
Thus, in a certain sense, $\oraclek(\cdot)$ can optimize with respect to $\conv(\alpha\mK)$ for all directions in $\reals^d_+$, although the oracle returns points in the original set $\mK$.

The following lemma shows that one can easily extend the oracle $\oraclek(\cdot)$ to optimize with respect to all directions in $\reals^d$. The extended approximation oracle described in the lemma forms the basis for both the algorithms in \cite{KKL} and the algorithms considered here.

\begin{lemma}[Extended approximation oracle]\label{lem:oracleExtand}
Given $\c\in\mathbb{R}^d$ write $\c = \c^+ + \c^-$ where $\c^+$ equals to $\c$ on all non-negative coordinates of $\c$ and zero everywhere else, and $\c^-$ equals $\c$ on all negative coordinates and zero everywhere else.
The \textit{extended approximation oracle} is a mapping $\oraclekext: \reals^d \rightarrow \left({\mK + \ball(0,(1+\alpha)R), ~ \mK}\right)$ defined as:
\begin{eqnarray}
\oraclekext(\c) = (\v,\s) := \left\{ \begin{array}{ll}
          \left({\oraclek(\c^+) - \alpha{}R\bar{\c}^-, ~ \oraclek(\c^+)}\right) & \mbox{if $\alpha \geq 1$} ;\\ \\
         \left({\oraclek(-\c^-) - R\bar{\c}^+, ~ \oraclek(-\c^-)}\right) & \mbox{if $\alpha < 1$},\end{array} \right. 
\end{eqnarray}
where for any vector $\v\in\reals^d$ we denote
\begin{eqnarray*}
\bar{\v} := \left\{ \begin{array}{ll}
          \v/\Vert{\v}\Vert & \mbox{if $\Vert{\v}\Vert > 0$} ;\\ \\
          \mathbf{0} & \mbox{if $\Vert{\v}\Vert = 0$},\end{array} \right.
\end{eqnarray*}
and it satisfies the following three properties:
\begin{enumerate}
\item $\v\cdot \c \leq \min_{\x\in\alpha\mK}\x\cdot \c$
\item $\forall{\f}\in\mF$: $\s \cdot \f \leq \v \cdot \f $ if $\alpha \geq 1$ and $\s \cdot \f \geq \v \cdot \f $ if $\alpha < 1$
\item $\Vert{\v}\Vert \leq (\alpha+2)R$
\end{enumerate}
\end{lemma}

\begin{proof}
For the first item in the lemma note that for $\alpha \geq 1$ it holds that
\begin{eqnarray*}
\v \cdot \c &=& \oraclek(\c^+) \cdot (\c^+ + \c^-) - \alpha{}R\bar{\c}^- \cdot (\c^+ + \c^-) \\
&\leq & \alpha\min_{\x\in\mK}\x \cdot \c^+ +\oraclek(\c^+)\cdot\c^-  - \alpha{}R\Vert{\c^-}\Vert \\
&\leq & \alpha\min_{\x\in\mK}\x \cdot \c^+ - \alpha{}R\Vert{\c^-}\Vert \\
&\leq & \alpha\min_{\x\in\mK}\x \cdot \c^+ + \alpha{}\min_{\x\in\mK}\x \cdot \c^- \\
&\leq & \alpha\min_{\x\in\mK}\x \cdot \c  = \min_{\x\in\alpha\mK}\x \cdot \c,
\end{eqnarray*}

Similarly, for $\alpha > 1$ we have that
\begin{eqnarray*}
\v \cdot \c &=& \oraclek(-\c^-) \cdot (\c^+ + \c^-) - R}{\bar{\c}^+ \cdot (\c^+ + \c^-) \\
&\leq & -\alpha\max_{\x\in\mK}\x \cdot (-\c^-) +\oraclek(-\c^-)\cdot\c^+  - R\Vert{\c^+}\Vert \\
&\leq & \alpha\min_{\x\in\mK}\x \cdot \c^- +R\Vert{\c^+}\Vert  - R\Vert{\c^+}\Vert \\
&\leq & \alpha\min_{\x\in\mK}\x \cdot \c^-  + \alpha\min_{\x\in\mK}\x \cdot \c^+ \\
&\leq & \alpha\min_{\x\in\mK}\x \cdot \c  = \min_{\x\in\alpha\mK}\x \cdot \c.
\end{eqnarray*}
For the second item, it suffices to observe that for $\alpha \geq 1$ we have that $\s \leq \v$ (coordinate-wise) and hence for every $\f\in\mF$ we have that $\s\cdot\f \leq \v\cdot \f$ (recall that $\mF\subset\reals^d_+$). Similarly, when $\alpha < 1$, we note that $\s \geq \v$. 

The third item holds trivially.
\end{proof}

It is important to note that while the extended oracle provides solutions with values at least as low as any point in $\conv(\alpha\mK)$, still in general the output point $\v$ need not be in either $\mK$ or $\conv(\alpha\mK)$, which means that it is not a feasible point to play in our OLO setting, nor does it allow us to optimize over $\conv(\alpha\mK)$. This is why we also need the oracle to output the feasible point $\s\in\mK$ which dominates $\v$ for any possible loss/payoff vector in $\mF$. While we will use the outputs $\v$ to solve a certain optimization problem involving $\conv(\alpha\mK)$, this dominance relation will be used to convert the solutions to these optimization problems into feasible plays for our OLO algorithms.

\subsubsection{Online gradient descent with and without feasibility}

As in \cite{KKL}, our online algorithms will be based on the well known \textit{Online Gradient Descent} method (OGD) for online convex optimization, originally due to \cite{Zinkevich03}. For a sequence of loss vectors $\{\f_1,...,\f_T\}\subset\reals^d$ OGD produces a sequence of plays $\{\x_1,...,\x_T\}\subset\mS$, for a convex and compact set $\mS\subset\reals^d$ via the following updates:
\begin{eqnarray*}
\forall t\geq 1: \qquad \y_{t+1}\gets \x_t - \eta\f_t, \quad \x_{t+1}\gets\arg\min_{\x\in\mS}\Vert{\x-\y_{t+1}}\Vert^2,
\end{eqnarray*}
where $\x_1$ is initialized to some arbitrary point in $\mS$ and $\eta$ is some pre-determined step-size. The obvious difficulty in applying OGD to online linear optimization over $\mS = \conv(\alpha\mK)$ is the step of computing $\x_{t+1}$ by projecting $\y_{t+1}$ onto the feasible set $\mS$, since as discussed, even with the extended approximation oracle, one cannot exactly optimize over $\conv(\alpha\mK)$. 
Instead we will consider a variant of OGD which may produce infeasible points, i.e., outside of $\mS$, but which guarantees low regret with respect to any point in $\mS$. This algorithm, which we refer to as \textit{online gradient descent without feasibility}, is given below (Algorithm \ref{alg:OGDinf}).

\begin{algorithm}[H]
\caption{Online Gradient Descent Without Feasibility}
\label{alg:OGDinf}
\begin{algorithmic}[1]
\STATE input: learning rate $\eta > 0$
\STATE $\x_1 \gets $ some point in $\mS$
\FOR{$t=1\dots T$}
\STATE play $\x_t$
\STATE receive loss/payoff vector $\f_t\in\reals^d$
\STATE $\y_{t+1} \gets \left\{ \begin{array}{ll}
         \x_t - \eta{}\f_t & \mbox{for losses} \\ 
          \x_t + \eta{}\f_t& \mbox{for payoffs}\end{array} \right. $
\STATE find $\x_{t+1}\in\reals^d$ such that 
\begin{eqnarray}\label{eq:infProj}
\forall \z\in\mS: \quad \Vert{\z-\x_{t+1}}\Vert^2 \leq \Vert{\z-\y_{t+1}}\Vert^2
\end{eqnarray}
(we say $\x_{t+1}$ is the infeasible projection of $\y_{t+1}$ onto the feasible set $\mS$)
\ENDFOR
\end{algorithmic}
\end{algorithm}

\begin{lemma}\label{lem:ogdRegret}[Online gradient descent without feasibility]
Fix $\eta > 0$. Suppose Algorithm \ref{alg:OGDinf} is applied for $T$ rounds and let $\{\f_t\}_{t=1}^T\subset\reals^d$ be the sequence of observed loss/payoff vectors, and let $\{\x_t\}_{t=1}^T$ be the sequence of points played by the algorithm. 
Then for any $\x\in\mS$ it holds that
\begin{eqnarray*}
 \begin{array}{ll}
          \frac{1}{T}\sum_{t=1}^T\x_t \cdot \f_t - \frac{1}{T}\sum_{t=1}^T\x\cdot \f_t \leq \frac{1}{2T\eta}\Vert{\x_1-\x}\Vert^2 + \frac{\eta{}}{2T}\sum_{t=1}^T\Vert{\f_t}\Vert^2 & \mbox{for losses} ;\\ \\
          \frac{1}{T}\sum_{t=1}^T\x\cdot \f_t - \frac{1}{T}\sum_{t=1}^T\x_t \cdot \f_t \leq \frac{1}{2T\eta}\Vert{\x_1-\x}\Vert^2 + \frac{\eta{}}{2T}\sum_{t=1}^T\Vert{\f_t}\Vert^2& \mbox{for payoffs}.\end{array}
\end{eqnarray*}
\end{lemma}

\begin{proof}
Fix $\x\in\mS$. Assume that the vectors $\f_1,...,\f_T$ are losses.
By the definition of the infeasible projection $\x_{t+1}$,
for any iteration $t\geq 1$ it holds that
\begin{eqnarray*}
\Vert{\x_{t+1}-\x}\Vert^2 &\leq & \Vert{\y_{t+1} - \x}\Vert^2 = \Vert{\x_t - \eta\f_t - \x}\Vert^2
\\
&=& \Vert{\x_t - \x}\Vert^2 - 2\eta(\x_t-\x)\cdot \f_t + \eta^2\Vert{\f_t}\Vert^2
\end{eqnarray*}
Rearranging and summing over all iterations we have that
\begin{eqnarray*}
\sum_{t=1}^T(\x_t-\x)\cdot \f_t &\leq & \frac{1}{2\eta}\sum_{t=1}^T(\Vert{\x_t-\x}\Vert^2 - \Vert{\x_{t+1}-\x}\Vert^2) + \frac{\eta}{2}\sum_{t=1}^T\Vert{\f_t}\Vert^2 \\
& \leq & \frac{1}{2\eta}\Vert{\x_1-\x}\Vert^2 + \frac{\eta}{2}\sum_{t=1}^T\Vert{\f_t}\Vert^2 .
\end{eqnarray*}

It is immediate to see that the proof of the result in case of payoffs instead of losses (for which the only change is in the update of $\y_{t+1}$ in Algorithm \ref{alg:OGDinf}), follows the same lines as the one for losses given above.
\end{proof}

\subsubsection{The KKL approach}\label{sec:KKL}

We now briefly describe how \cite{KKL} use the extended approximation oracle and the online gradient descent without feasibility approach to construct their low $\alpha$-regret algorithm for the full information setting, and point out the limitation of this approach to obtaining low oracle complexity.

Consider some iteration $t$ of Algorithm \ref{alg:OGDinf} and let $\y_{t+1}$ be the newly computed point. Let $(\x,\s)\in\reals^d\times\mK$ be such that $\forall\f\in\mF: \x\cdot\f \geq \s\cdot\f$ (e.g., take $\x = \s$), and let $(\v',\s') \gets \oraclekext(\x-\y_{t+1})$. We have the following simple lemma, a proof of which is given in the appendix for completeness.

\begin{lemma}\label{lem:KKL:1}
Fix $\epsilon\in(0, 3(\alpha+2)^2R^2]$ and suppose that $\x\in\ball(0,(\alpha+2)R)$. If $(\x- \y_{t+1})\cdot(\x - \v') \leq \epsilon$, then setting $\x_{t+1} \gets \x$ gives
\begin{eqnarray*}
\forall \z\in\conv(\alpha\mK): \quad \Vert{\z-\x_{t+1}}\Vert^2 \leq \Vert{\z-\y_{t+1}}\Vert^2 + 2\epsilon.
\end{eqnarray*}
Otherwise, setting $\x' \gets (1-\lambda)\x + \lambda\v'$, for appropriately chosen $\lambda\in(0,1)$, guarantees that
\begin{eqnarray*}
\Vert{\x' - \y_{t+1}}\Vert^2 \leq \Vert{\x - \y_{t+1}}\Vert^2 - \Omega(\epsilon^2),
\end{eqnarray*}
and
\begin{eqnarray*}
\forall \f\in\mF: \qquad \left({(1-\lambda)\s + \lambda\s'}\right)\cdot\f \leq \x'\cdot\f.
\end{eqnarray*}
\end{lemma}

Note that Lemma \ref{lem:KKL:1} suggests an iterative algorithm to compute an $\epsilon$-approximated projection of $\y_{t+1}$ in Algorithm \ref{alg:OGDinf}, that on each iteration reduces the potential $\Vert{\x-\y_{t+1}}\Vert^2$ by $\Omega(\epsilon^2)$, until finding an $\epsilon$-approximated projection of $\y_{t+1}$, $\x_{t+1}$, which must be found since the potential in non-negative. Moreover, this algorithm finds a point $\bar{\s}_{t+1}\in\conv(\mK)$, given explicitly as a convex combination of points in $\mK$ (since $\lambda\in(0,1)$), such that $\bar{\s}_{t+1}$ dominates $\x_{t+1}$ for all vectors in $\mF$. In particular, sampling $\s_{t+1}$ from this decomposition guarantees that we play a feasible point in $\mK$, which in expectation, dominates $\x_{t+1}$ for all vectors in $\mF$. The full algorithm, which is closely related to the classical Frank-Wolfe algorithm for convex optimization (a.k.a. the conditional gradient method) \cite{Bubeck15}, is given for completeness in the appendix, see Algorithm \ref{alg:fw}\footnote{we note that it differs somewhat in presentation than the original algorithm in \cite{KKL}.}.

The proof of the following lemma is also given in the appendix for completeness.
\begin{lemma}\label{lem:KKL:2}
Fix $\epsilon\in(0, 3(\alpha+2)^2R^2]$, $\eta > 0$ and a sequence of loss functions $\{\f_1,...,\f_T\}\subseteq\mF$. Consider the application of Algorithm \ref{alg:OGDinf} with learning rate $\eta$ when applied with respect to the feasible set $\conv(\alpha\mK)$ and the sequence of losses $\{\f_1,...,\f_T\}\subseteq\mF$, and when we use the algorithm described above to produce the (randomized) sequence of points $\{(\x_t,\s_t\}_{t\in[T]}\subset\reals^d\times\mK$. Then, focusing on the case $\alpha \geq 1$, it holds that,
\begin{eqnarray*}
\E\left[{\frac{1}{T}\sum_{t=1}^T\s_t \cdot \f_t}\right] - \min_{\x\in\conv(\alpha\mK)}\frac{1}{T}\sum_{t=1}^T\x\cdot \f_t &=&
\E\left[{\frac{1}{T}\sum_{t=1}^T\s_t \cdot \f_t}\right] - \alpha\min_{\x\in\mK}\frac{1}{T}\sum_{t=1}^T\x\cdot \f_t \\
&\leq& \frac{\alpha^2R^2}{T\eta} + \frac{\eta{}F^2}{2}  + \frac{\epsilon}{\eta}.
\end{eqnarray*}
Moreover, the number of calls to the extended approximation oracle per iteration $t$ is $$O(\Vert{\y_{t+1}-\x_t}\Vert^2/\epsilon^2) = O(\eta^2F^2/\epsilon^2),$$where the $O(\cdot)$ notation hides polynomial dependencies on $(1+\alpha),R$.
\end{lemma}

The extra term of $\epsilon/\eta$ in the regret bound is due to fact we compute $\epsilon$-approximated projections.

It is clear that setting $\eta = O(1/\sqrt{T})$ and $\epsilon = O(1/T)$ in Lemma \ref{lem:KKL:2} guarantees $O(T^{-1/2})$ expected $\alpha$-regret, which is optimal in $T$, however requires $O(T)$ calls to the approximation oracle per iteration.
We can also observe that for any constants $a\in(0,1),b\geq 1$, and sufficiently large $T$, Lemma \ref{lem:KKL:2} cannot guarantee $O(T^{-a})$ expected $\alpha$-regret using only $O(\log^bT)$ calls to the approximation oracle per iteration, even on average. For this reason, in this paper we consider a drastically different algorithmic approach to applying the online gradient descent without feasibility methodology.

\section{Oracle-efficient Computation of (infeasible) Projections onto $\conv(\alpha\mK)$}\label{sec:project}

In this section we detail our main technical tool for obtaining oracle-efficient online algorithms, i.e., our algorithm for computing projections, in the sense of Eq. \eqref{eq:infProj}, onto the convex set $\conv(\alpha\mK)$. 
Before presenting our projection algorithm, Algorithm \ref{alg:approxProj} and detailing its theoretical guarantees, we first present the main algorithmic building block in the algorithm, which is described in the following lemma. Lemma \ref{lem:sepORdec} shows that for any point $\x\in\reals^d$, we can either find a near-by point $\p$ which is a convex combination of points outputted by the extended approximation oracle (and hence, $\p$ is dominated by a convex combination of feasible points in $\mK$ for any vector in $\mF$, as discussed in Section \ref{sec:extended_oracle}), or we can find a separating hyperplane that separates $\x$ from $\conv(\alpha\mK)$ with sufficiently large margin. We achieve this by running the well known Ellipsoid method \cite{GLS81,Bubeck15} in a very specialized way. This application of the Ellipsoid method is similar in spirit to those in \cite{Papadimitriou08, Weinberg14}, which applied this idea to computing \textit{correlated equilibrium} in games and \textit{algorithmic mechanism design}, though the implementation details and the way in which we apply this technique are quite different.

\begin{lemma}[Separation-or-Decomposition via the Ellipsoid method]\label{lem:sepORdec}
Fix $\x\in\reals^d$, $ \epsilon\in\left({0,~(\alpha+2)R}\right]$, and a positive integer $N \geq cd^2\ln\left({\frac{(\alpha+1)R+\Vert{\x}\Vert}{\epsilon}}\right)$, where $c$ is a positive universal constant. Consider an attempt to apply the Ellipsoid method for $N$ iterations to the following feasibility problem:
\begin{eqnarray}\label{eq:sepORdec:prob}
\textrm{find}~ \w\in\reals^d ~ \textrm{such that:} \quad&  \forall \z\in\alpha\mK: \quad (\x-\z)\cdot \w \geq \epsilon, & \nonumber \\
&\Vert{\w}\Vert \leq 1 ,&
\end{eqnarray}
such that each iteration of the Ellipsoid method applies the following consecutive steps:
\begin{enumerate}
\item
$(\v,\s) \gets \oraclekext(-\w)$, where $\w$ is the current iterate. If $(\x-\v)\cdot \w < \epsilon$, use $\v-\x$ as a separating hyperplane for the Ellipsoid method and continue to to the next iteration
\item 
if $\Vert{\w}\Vert > 1$, use $\w$ as a separating hyperplane for the Ellipsoid method and continue to the next iteration
\item 
otherwise ($\Vert{\w}\Vert \leq 1$ and $(\x-\v)\cdot \w \geq \epsilon$), declare Problem \eqref{eq:sepORdec:prob} \textit{feasible} and return the vector $\w$.
\end{enumerate}

Then, if the Ellipsoid method terminates declaring Problem \ref{eq:sepORdec:prob} \textit{feasible},  the returned vector $\w$ is a feasible solution to Problem \eqref{eq:sepORdec:prob}.
Otherwise (the Ellipsoid method completes $N$ iterations without declaring Problem \eqref{eq:sepORdec:prob} feasible), let $(\v_1,\s_1),...,(\v_{N},\s_{N})$ be the outputs of the extended approximation oracle gathered throughout the run of the algorithm, and let $(a_1,...,a_N)$ be an optimal solution to the following optimization problem:
\begin{eqnarray}\label{eq:lem:sepORdec:proj}
&\min \frac{1}{2}\left\|{\sum_{i=1}^{N}a_i\v_i - \x}\right\|^2 & \nonumber \\
\textrm{s.t.} & \forall i\in\{1,...,N\}: ~ a_i \geq 0, \quad \sum_{i=1}^{N}a_i =1 &.
\end{eqnarray}
Then the point $\p = \sum_{i=1}^{N}a_i\v_i$ satisfies $\Vert{\x-\p}\Vert \leq 3\epsilon$.
\end{lemma}

\begin{proof}
To prove the first part of the lemma, suppose there exists some iteration during which the Ellipsoid method declares Problem \eqref{eq:sepORdec:prob} feasible, and let $\w$ be the corresponding iterate and let $(\v,\s)$ be the output of the extended approximation oracle on that iteration. Clearly it holds that
\begin{eqnarray*}
\epsilon &\leq & (\x-\v)\cdot \w =  \x\cdot \w + \v \cdot (-\w) \leq \x\cdot \w + \min_{\z\in\alpha\mK}\z\cdot(-\w) =
\min_{\z\in\alpha\mK}(\x-\z)\cdot \w,
\end{eqnarray*}
where the first inequality follows from the fact that the Ellipsoid method declared Problem \eqref{eq:sepORdec:prob} feasible, and the second inequality follows from the definition of the extended approximation oracle. 
Since the Ellipsoid method declared Problem \eqref{eq:sepORdec:prob} feasible, it also follows that $\Vert{\w}\Vert \leq 1$ and hence $\w$ is indeed a feasible solution to Problem \eqref{eq:sepORdec:prob}.

Consider now the case that all $N$ iterations are executed without declaring Problem \eqref{eq:sepORdec:prob} feasible and let $\v_1,...,\v_N'$ be as defined in the lemma. We would like to show that this implies that 
\begin{eqnarray}\label{eq:lem:sepORdec:1}
\forall~\textrm{unit vector} ~ \w: \quad \min_{i\in\{1,...,N\}}(\x-\v_i)\cdot \w \leq 3\epsilon.
\end{eqnarray}
Then, the second part of the lemma follows from applying the next lemma, Lemma \ref{lem:goodProj}, which shows that \eqref{eq:lem:sepORdec:1}  implies that the point $\p$ defined in the lemma indeed satisfies $\Vert{\p-\x}\Vert \leq 3\epsilon$, as required.

Towards proving \eqref{eq:lem:sepORdec:1}, suppose that there exists a unit vector $\h\in\reals^d$ such that for all $i\in\{1,...,N\}$, $(\x-\v_i)\cdot \h > 3\epsilon$. It follow that $\forall i\in\{1,...,N\}: (\x-\v_i)\cdot \h/2 > 3\epsilon/2$. It follows from a simple application of the Cauchy-Swartz inequality and the observation that $\Vert{\x-\v_i}\Vert \leq \Vert{\x}\Vert + \Vert{\v_i}\Vert \leq \Vert{\x}\Vert + (\alpha+2)R$, that denoting $r := \frac{\epsilon}{2(\alpha+2)R+\Vert{\x}\Vert}$, we have that 
\begin{eqnarray}\label{eq:lem:sepORdec:2}
\forall \h'\in\ball(\h/2, r): \quad \min_{i\in[N]}(\x-\v_i)\cdot \h' > \epsilon.
\end{eqnarray}
Note that on one hand, by the above and our assumption on $\epsilon$, every point in $\ball(\h/2, r)$ satisfies the stopping criteria of the Ellipsoid method described in the lemma.  On the other-hand, on every iteration in which the current iterate $\w$ is not declared feasible, it follows that the separating hyperplane fed to the Ellipsoid method indeed separates $\w$ from $\ball(\h/2, r)$. To see why this is true, we consider the two possible options for the separating hyperplane. If the hyperplane is $\v_i-\x$, where $\v_i$ is the output of the extended approximation oracle on that iteration, then we have that
\begin{eqnarray*}
\forall \h'\in\ball(\h/2,r): \quad (\w - \h')\cdot(\v_i-\x) = (\x-\v_i)\cdot\h' - (\x-\v_i)\cdot\w > \epsilon - \epsilon = 0,
\end{eqnarray*}
where the first inequality follows from Eq. \eqref{eq:lem:sepORdec:2} and the fact that $(\x-\v_i)\cdot\w < \epsilon$ on this iteration. If the hyperplane used was $\w$, which guarantees that on that iteration $\Vert{\w}\Vert > 1$ , then we have that
\begin{eqnarray*}
\forall \h'\in\ball(\h/2,r): \quad (\w - \h')\cdot\w = \Vert{\w}\Vert  - \w\cdot\h' > 1 -1 = 0,
\end{eqnarray*}
where the last inequality follows since by our assumption on $\epsilon$, it holds that $\ball(\h/2,r)\subset\ball(0,1)$.
Thus, we can conclude that if the number of Ellipsoid method iterations satisfies $N \geq cd^2\ln\left({\frac{(\alpha+1)R+\Vert{\x}\Vert}{\epsilon}}\right)$ for an appropriate universal constant $c > 0$, and all $N$ iterations were completed without declaring feasibility, it follows that no such unit vector $\h$ can exist, which means Eq. \eqref{eq:lem:sepORdec:1} holds, and the result follows.
\end{proof}

\begin{lemma}\label{lem:goodProj}
Fix $\x\in\reals^d$, vectors $\v_1,...,\v_N\in\reals^d$ and $\epsilon > 0$. If for any unit vector $\w$ it holds that $\min_{i\in\{1,...,N\}}(\x-\v_i)\cdot \w \leq \epsilon$, then it follows that the point $\p=\sum_{i=1}^Na_i\v_i$, where $(a_1,...,a_N)$ is an optimal solution to Problem \eqref{eq:lem:sepORdec:proj}, satisfies $\Vert{\p-\x}\Vert \leq \epsilon$. 
\end{lemma}

\begin{proof}
First we show that the following holds:
\begin{eqnarray}\label{eq:lem:goodProj:1}
&\forall {i,j}~ \textrm{s.t.} ~ a_i > 0, a_j > 0: \quad (\p-\x)\cdot \v_i = (\p-\x)\cdot \v_j, & \nonumber \\
&\forall{i,j}~ \textrm{s.t.} ~ a_i > 0, a_j = 0: \quad (\p-\x)\cdot \v_i \leq (\p-\x)\cdot \v_j. &
\end{eqnarray}
To see why this is true, fix some $i,j$ such that $a_i > 0$ and consider the point $\p' = \p + \delta(\v_j-\v_i)$ such that $0 < \delta \leq a_i$. Clearly $\p'$ lies in the convex hull of $\{\v_1,...,\v_N\}$ and hence is a feasible solution to Problem \eqref{eq:lem:sepORdec:proj}.
It holds that
\begin{eqnarray}\label{eq:lem:goodProj:2}
\frac{1}{2}\Vert{\p' -\x}\Vert^2 = \frac{1}{2}\Vert{\p -\x}\Vert^2 + \delta(\v_j-\v_i)\cdot(\p-\x) + \frac{\delta^2}{2}\Vert{\v_i-\v_j}\Vert^2.
\end{eqnarray}

Thus, we can see that if \eqref{eq:lem:goodProj:1} does not hold, then without loss of generality we can always choose $i,j$ such that $a_i > 0$ and $(\p-\x)\cdot \v_i > (\p-\x)\cdot \v_j$, and thus as can be seen from Eq. \eqref{eq:lem:goodProj:2}, choosing $\delta$ to be sufficiently small it follows that $\Vert{\p'-\x}\Vert^2 < \Vert{\p-\x}\Vert^2$, contradicting the optimality of $\p$.

Denoting by $\u$ the unit vector in the direction of $\x-\p$, we can rewrite Eq. \eqref{eq:lem:goodProj:1} as follows:
\begin{eqnarray}\label{eq:lem:goodProj:3}
&\forall {i,j}~ \textrm{s.t.} ~ a_i > 0, a_j > 0: \quad (\x-\v_i)\cdot \u = (\x-\v_j) \cdot \u, & \nonumber \\
&\forall{i,j}~ \textrm{s.t.} ~ a_i > 0, a_j = 0: \quad (\x-\v_i)\cdot \u \leq (\x-\v_j)\cdot \u. &
\end{eqnarray}

Using our assumption, we in particular have that $\min_{i\in[N]}(\x-\v_i)\cdot \u \leq \epsilon$, and using Eq. \eqref{eq:lem:goodProj:3} we have that

\begin{eqnarray*}
\Vert{\p-\x}\Vert = (\x-\p)\cdot \u = \sum_{i=1}^Na_i(\x-\v_i)\cdot \u = \min_{i\in[N]}(\x-\v_i)\cdot \u \leq \epsilon,
\end{eqnarray*}
where the last equality is a consequence of Eq.  \eqref{eq:lem:goodProj:3} and the fact that $(a_1,...,a_N)$ is a distribution.
Thus the lemma follows.
\end{proof}

We are now ready to present our algorithm for computing projections onto $\conv(\alpha\mK)$ (in the sense of Eq. \eqref{eq:infProj}). Consider now an attempt to project a point $\y\in\reals^d$, and note that in particular, $\y$ itself is a valid projection (again, in the sense of Eq. \eqref{eq:infProj}), however, in general, it is not a feasible point nor is it dominated by a convex combination of feasible points. When attempting to project $\y\in\reals^d$, our algorithm continuously applies the \textit{separation-or-decomposition} procedure described in Lemma \ref{lem:sepORdec}. In case the procedure returns a decomposition, then by Lemma \ref{lem:sepORdec}, we have a point that is sufficiently close to $\y$ and is dominated for any vector in $\mF$ by a convex combination (given explicitly) of feasible points in $\mK$. Otherwise, the procedure returns a separating hyperplane which can be used to to ``pull $\y$ closer'' to $\conv(\alpha\mK)$ in a way that the resulting point still satisfies the projection inequality given in Eq. \eqref{eq:infProj}, and the process then repeats itself. Since each time we obtain a hyperplane separating our current iterate from $\conv(\alpha\mK)$, we pull the current iterate sufficiently towards $\conv(\alpha\mK)$, this process must terminate. Lemma \ref{lem:approxProj} gives exact bounds on the performance of the algorithm.

\begin{algorithm}
\caption{(infeasible) Projection onto $\conv(\alpha\mK)$}
\label{alg:approxProj}
\begin{algorithmic}[1]
\STATE input: point $\y\in\reals^d$, tolerance $\epsilon > 0$
\STATE $\tilde{\y} \gets \y / \max\{1,~ \Vert{\y}\Vert/(\alpha{}R)\}$
\FOR{$t=1\dots$}
\STATE call the \textsc{Separation-or-Decompostion} procedure described in Lemma \ref{lem:sepORdec} with parameters $(\tilde{\y}, \epsilon)$
\IF{the procedure outputs a separating hyperplane $\w$}
\STATE $\tilde{\y} \gets \tilde{\y} - \epsilon\w$
\ELSE
\STATE let $(a_1,...,a_N)$, $\{(\v_1,\s_1),...,(\v_N,\s_N)\}$ be the decomposition returned, where $N$ is as prescribed in Lemma \ref{lem:sepORdec}.
\RETURN $\tilde{\y}$, $(a_1,...,a_N)$, $\{(\v_1,\s_1),...,(\v_N,\s_N)\}$ 
\ENDIF
\ENDFOR
\end{algorithmic}
\end{algorithm}

\begin{lemma}\label{lem:approxProj}
Fix $\y\in\reals^d$ and $\epsilon \in (0,~(\alpha+2)R]$. Algorithm \ref{alg:approxProj} terminates after at most $\lceil{\alpha^2R^2 / \epsilon^2}\rceil$ iterations, returning a point $\tilde{\y}\in\reals^d$, a distribution $(a_1,...,a_N)$ and a set $\{(\v_1,\s_1),...,(\v_N,\s_N)\}$ outputted by the extended approximation oracle, where $N$ is as defined in Lemma \ref{lem:sepORdec}, such that
\begin{enumerate}
\item
$\forall \z\in\conv(\alpha\mK): \quad \Vert{\tilde{\y}-\z}\Vert^2 \leq \Vert{\y-\z}\Vert^2$,
\item
$\Vert{\p - \tilde{\y}}\Vert \leq 3\epsilon$ for $\p := \sum_{i=1}^Na_i\v_i$.
\end{enumerate}
Moreover, if the \textbf{for} loop was entered a total number of $k$ times, then the final value of $\tilde{\y}$ satisfies 
\begin{eqnarray*}
\dist^2(\tilde{\y},\conv(\alpha\mK)) \leq \min\{2\alpha^2R^2,~ \dist^2(\y,\conv(\alpha\mK)) - (k-1)\epsilon^2\}, 
\end{eqnarray*}
and the overall number of queries to the approximation oracle is $O\left({kd^2\ln\left({\frac{(\alpha+1)R}{\epsilon}}\right)}\right)$.
\end{lemma}

\begin{proof}
Note that the second item in the lemma is a straightforward guarantee of Lemma \ref{lem:sepORdec}. 

To prove the first item, suppose that the algorithm terminates after the \textbf{for} loop was entered $k$ times, and let $\tilde{\y}_1,...,\tilde{\y}_k$ denote the values of $\tilde{\y}$ throughout the run of the algorithm, where $\tilde{\y}_i$ is the value of $\tilde{\y}$ at the beginning of the $i$th iteration of the \textbf{for} loop. Note that since $\conv(\alpha\mK)\subseteq\ball(0,\alpha{}R)$ and $\tilde{\y}_1$ is the projection of $\y$ onto $\ball(0,\alpha{}R)$, we have that $\forall \z\in\conv(\alpha\mK): \Vert{\tilde{\y}_1 - \z}\Vert^2 \leq \Vert{\y-\z}\Vert^2$.

We are now going to show that for any $i\geq 1$ it holds that 
\begin{eqnarray}\label{eq:lem:approxProj:1}
\forall \z\in\conv(\alpha\mK): \quad \Vert{\tilde{\y}_{i+1}-\z}\Vert^2 \leq \Vert{\tilde{\y}_{i}-\z}\Vert^2,
\end{eqnarray}
which clearly yields item 1 in the lemma.

To prove that Eq. \eqref{eq:lem:approxProj:1} holds throughout the run of the algorithm, consider an iteration $i$ of the \textbf{for} loop during which, the \textsc{Separation-or-Decompostion} procedure returns a separating hyperplane $\w$. It holds that
\begin{eqnarray*}
\forall \z\in\conv(\alpha\mK):\quad \Vert{\tilde{\y}_{i}-\z}\Vert^2 &=& \Vert{\tilde{\y}_{i}-\tilde{\y}_{i+1} + \tilde{\y}_{i+1}-\z}\Vert^2 \\
&=& \Vert{\tilde{\y}_{i}-\tilde{\y}_{i+1}}\Vert^2  + \Vert{\tilde{\y}_{i+1}-\z}\Vert^2 
+ 2(\tilde{\y}_{i}-\tilde{\y}_{i+1}) \cdot (\tilde{\y}_{i+1}-\z) \\
&\geq & \Vert{\tilde{\y}_{i+1}-\z}\Vert^2 
+ 2(\tilde{\y}_{i}-\tilde{\y}_{i+1}) \cdot (\tilde{\y}_{i+1}-\z) \\
&= & \Vert{\tilde{\y}_{i+1}-\z}\Vert^2 +
2\epsilon\w \cdot [(\tilde{\y}_{i} - \z) - \epsilon\w] \\
&= & \Vert{\tilde{\y}_{i+1}-\z}\Vert^2 +
2\epsilon(\tilde{\y}_{i} - \z)\cdot \w  - 2\epsilon^2\Vert{\w}\Vert^2 \\
&\geq & \Vert{\tilde{\y}_{i+1}-\z}\Vert^2,
\end{eqnarray*}
where the third equality follows from the update rule of $\tilde{\y}$ in the algorithm, and the last inequality is a direct consequence of the guarantees of Lemma \ref{lem:sepORdec}. Thus, Eq. \eqref{eq:lem:approxProj:1} indeed holds for all $i\geq 1$, which gives the first item listed in the lemma.

We now turn to upper bound the number of iterations performed by the algorithm. Consider again an iteration $i$ of the loop during which the \textsc{Separation-or-Decompostion} procedure returns a separating hyperplane $\w$. We are going to show that
\begin{eqnarray*}
\dist^2(\tilde{\y}_{i+1}, \conv(\alpha\mK)) \leq \dist^2(\tilde{\y}_{i}, \conv(\alpha\mK)) - \epsilon^2,
\end{eqnarray*}
which, together with the fact that  $\dist^2(\tilde{\y}_{1}, \conv(\alpha\mK)) \leq 2\alpha^2R^2$, gives the desired upper bound on the number of iterations.

Denote $\x_{i} = \arg\min_{\x\in\conv(\alpha\mK)}\Vert{\x - \tilde{\y}_{i}}\Vert$ and $\x_{i+1} = \arg\min_{\x\in\conv(\alpha\mK)}\Vert{\x - \tilde{\y}_{i+1}}\Vert$. It holds that
\begin{eqnarray*}
\dist^2(\tilde{\y}_{i+1}, \conv(\alpha\mK)) &=& \Vert{\x_{i+1} - \tilde{\y}_{i+1}}\Vert^2 \leq \Vert{\x_{i} - \tilde{\y}_{i+1}}\Vert^2
= \Vert{\x_{i} - \tilde{\y}_{i} + \epsilon\w}\Vert^2 \\
&=& \dist^2(\tilde{\y}_{i}, \conv(\alpha\mK)) +\epsilon^2\Vert{\w}\Vert^2 - 2\epsilon(\tilde{\y}_{i}-\x_{i})\cdot \w \\
&\leq& \dist^2(\tilde{\y}_{i}, \conv(\alpha\mK)) -\epsilon^2,
\end{eqnarray*}
where the inequality is a direct consequence of the guarantees of Lemma \ref{lem:sepORdec}. Thus, we obtain both the desired bound on the number of iterations and the bound on the distance of the final point $\tilde{\y}$ from $\conv(\alpha\mK)$.

Finally, we turn to upper bound to overall number of queries to the approximation oracle. Using the bound in Lemma \ref{lem:sepORdec}, we have that the number of calls to the oracle on the $i$th iteration of the loop is upper bounded by $O\left({d^2\ln\left({\frac{(\alpha+1)R+\Vert{\tilde{\y}_i}\Vert}{\epsilon}}\right)}\right)$. As we have shown, the values $\dist(\tilde{\y}_i, \conv(\alpha\mK))$ are monotonically decreasing with $i$ and hence we can upper bound
\begin{eqnarray*}
\Vert{\tilde{\y}_i}\Vert \leq \max_{\x\in\conv(\alpha\mK)}\Vert{\x}\Vert + \dist(\tilde{\y}_i, \conv(\alpha\mK)) \leq \alpha{}R + \dist(\tilde{\y}_1, \conv(\alpha\mK))
\leq \alpha{}R + \sqrt{2}\alpha{}R,
\end{eqnarray*}
where the last inequality holds since $\tilde{\y}_1$ is the projection of $\y$ onto the ball $\ball(0,\alpha{}R)$.
Thus, the overall number of queries to the approximation oracle after $k$ iterations is upper bounded by $O\left({kd^2\ln\left({\frac{(\alpha+1)R}{\epsilon}}\right)}\right)$.
\end{proof}

It is important to note that the worst case iteration bound in Lemma \ref{lem:approxProj} does not seem so appealing for our purposes, since it depends polynomially on $1/\epsilon$, and in our online algorithms naturally we will need to take $\epsilon = O(T^{-c})$ for some $c>0$, which seems to contradict our goal of achieving poly-logarithmic in $T$ oracle complexity, at least on average. However, as Lemma \ref{lem:approxProj} shows, the more iterations Algorithm \ref{alg:approxProj} performs, the closer it brings its final iterate to the set $\conv(\alpha\mK)$. Thus, as we will show when analyzing the oracle complexity of our online algorithms, while a single call to Algorithm \ref{alg:approxProj} can be expensive, when calling it sequentially, where each input is a small perturbation of the output of the previous call, the average number of iterations performed per such call cannot be too high.

\paragraph{Remark regarding use of the Ellipsoid method:} while our core algorithmic tool, the \textit{seperation-or-decomposition} procedure, described in Lemma \ref{lem:sepORdec}, is based on an application of the Ellipsoid method, it should be clear that any other optimization method with a similar ``optimization interface" (i.e., is based on separation queries and guarantees similar bounds on the number of iterations required to find a feasible point), may be used. Here we chose to use the Ellipsoid method because it is the most well known. Using a different method might potentially improve the explicit dependence of our complexity bounds on the dimension $d$ from quadratic to almost linear.

\section{Efficient Algorithms for the Full Information and Bandit Settings}\label{sec:algorithms}

We now turn to present our online algorithms for the full-information and bandit settings together with their exact regret bounds and oracle-complexity guarantees.

\subsection{Algorithm for the full information setting}

Our algorithm for the full-information setting, Algorithm \ref{alg:newOGD}, is given below.

\begin{algorithm}
\caption{Online Gradient Descent with Infeasible Projections onto $\conv(\alpha\mK)$}
\label{alg:newOGD}
\begin{algorithmic}[1]
\STATE input: learning rate $\eta > 0$, projection error parameter $\epsilon > 0$
\STATE $\s_1 \gets $ some point in $\mK$
\STATE $\tilde{\y}_1 \gets \alpha\s_1$
\FOR{$t=1\dots T$}
\STATE play $\s_t$
\STATE receive loss/payoff vector $\f_t\in\mF$
\STATE $\y_{t+1} \gets \left\{ \begin{array}{ll}
         \tilde{\y_t} - \eta{}\f_t & \mbox{if $\alpha \geq 1$} \\ 
          \tilde{\y_t} + \eta{}\f_t& \mbox{if $\alpha < 1$}\end{array} \right. $
\STATE call Algorithm \ref{alg:approxProj} with inputs $(\y_{t+1}, \epsilon)$ to obtain an approximated projection $\tilde{\y}_{t+1},$ a distribution $(a_1,...,a_{N})$ and $\{(\v_1,\s_1),...,(\v_{N},\s_{N})\}\subseteq\reals^d\times\mK$, for some $N\in\mathbb{N}$.
\STATE sample $\s_{t+1}\in\{\s_1,...,\s_{N}\}$ according to distribution $(a_1,...,a_{N})$
\ENDFOR
\end{algorithmic}
\end{algorithm}

\begin{theorem}\label{thm:newOGD}[Main Theorem]
Fix $\eta > 0, \epsilon \in (0,~(\alpha+2)R]$. Suppose Algorithm \ref{alg:newOGD} is applied for $T$ rounds and let $\{\f_t\}_{t=1}^T\subseteq\mF$ be the sequence of observed loss/payoff vectors, and let $\{\s_t\}_{t=1}^T$ be the sequence of points played by the algorithm.
Then it holds that
\begin{eqnarray*}
\E\left[{\alpharegret}\left({\{(\s_t,\f_t)\}_{t\in[T]}}\right)\right] \leq  \frac{\alpha^2R^2}{T\eta} + \frac{\eta{}F^2}{2} + 3F\epsilon,
\end{eqnarray*}
and the average number of calls to the approximation oracle of $\mK$ per iteration is upper bounded by
\begin{eqnarray*}
K(\eta, \epsilon) := O\left({\left({1+  \frac{\eta\alpha{}RF + \eta^2F^2}{\epsilon^2}}\right)d^2\ln\left({\frac{(\alpha+1)R}{\epsilon}}\right)}\right).
\end{eqnarray*}
In particular, setting $\eta = \alpha{}RT^{-2/3}/F$, $\epsilon = \alpha{}RT^{-1/3}$ gives
\begin{eqnarray*}
\E\left[{\alpharegret}\right] = O\left({\alpha{}RFT^{-1/3}}\right), \quad K = O\left({d^2\ln\left({\frac{\alpha+1}{\alpha}T}\right)}\right).
\end{eqnarray*}
Alternatively, setting $\eta = \alpha{}RT^{-1/2}/F$, $\epsilon = \alpha{}RT^{-1/2}$ gives
\begin{eqnarray*}
\E\left[{\alpharegret}\right] = O\left({\alpha{}RFT^{-1/2}}\right), \quad K = O\left({\sqrt{T}d^2\ln\left({\frac{\alpha+1}{\alpha}T}\right)}\right).
\end{eqnarray*}
\end{theorem}

\begin{proof}
For the proof we focus on the case $\alpha \geq 1$ since the proof for the complementary follows from the same derivations up to changes in the obvious places.
To prove the regret bound, we simply apply Lemma \ref{lem:ogdRegret} with respect to the sequence of points $\{\tilde{\y}_t\}_{t=1}^T$ and the feasible set $\conv(\alpha\mK)$ and plugin the guarantee of Lemma \ref{lem:approxProj}, which gives
\begin{eqnarray*}
\sum_{t=1}^T\tilde{\y}_t\cdot\f_t - \min_{\x\in\alpha\mK}\sum_{t=1}^T\x\cdot\f_t &=& \sum_{t=1}^T\tilde{\y}_t\cdot\f_t - \alpha\cdot\min_{\x\in\mK}\sum_{t=1}^T\x\cdot\f_t \\
&  \leq & \frac{\alpha^2R^2}{\eta} + T\frac{\eta{}F^2}{2},
\end{eqnarray*}
where we have used the the fact that $\Vert{\tilde{\y}_1}\Vert \leq \alpha{}R$ and  $\Vert{\f_t}\Vert \leq F$ for all $t\in[T]$.
For every iteration $t\geq 1$, let us denote $\p_{t+1} = \sum_{i=1}^{N}a_i\v_i$, $\bar{\s}_t = \sum_{i=1}^{N}a_i\s_i$, where $(a_1,...,a_{N})$, $\{(\v_1,\s_1),...,(\v_{N},\s_{N})\}$ are the outputs of the call to Algorithm \ref{alg:approxProj} on iteration $t$, and for $t=1$ we denote $\p_1 = \tilde{\y}_1$ and $\bar{\s}_1 = \s_1$. By the guarantee of Lemma \ref{lem:approxProj}, we have that
\begin{eqnarray*}
\sum_{t=1}^T\p_t\cdot\f_t - \alpha\cdot\min_{\x\in\mK}\sum_{t=1}^T\x\cdot\f_t \leq  \frac{\alpha^2R^2}{\eta} + T\frac{\eta{}F^2}{2} + 3T\epsilon{}F,
\end{eqnarray*}
where the inequality holds since for all $t\geq 1$: $\vert{(\p_t - \tilde{\y}_t)\cdot \f_t}\vert \leq \Vert{\p_t - \tilde{\y}_t}\Vert\cdot\Vert{\f_t}\Vert \leq 3\epsilon{}F$. The regret bound now follows since for any iteration $t$, $\bar{\s}_t$ dominates $\p_t$ for any vector $\f\in\mF$, and since $\E[\s_t] = \bar{\s}_t$.

We now turn to upper bound the overall number of queries to the approximation oracle of $\mK$. Let $k_t$ be the number of iterations it took Algorithm \ref{alg:approxProj} to terminate, when invoked on iteration $t$ of Algorithm \ref{alg:newOGD}. Note that, by Lemma \ref{lem:approxProj}, we have that $K(\eta, \epsilon) = O\left({\frac{1}{T}\sum_{t=1}^{T-1}k_td^2\ln\left({\frac{(\alpha+1)R}{\epsilon}}\right)}\right)$. By Lemma \ref{lem:approxProj}, it follows that on any iteration $t$, 
\begin{eqnarray*}
\dist^2(\tilde{\y}_{t+1}, \conv(\alpha\mK)) &\leq& \dist^2(\y_{t+1}, \conv(\alpha\mK)) - (k_t-1)\epsilon^2 \\
 &=& \dist^2(\tilde{\y}_{t}-\eta\f_{t}, \conv(\alpha\mK)) - (k_t-1)\epsilon^2 \\
&\leq& (\dist(\tilde{\y}_{t}, \conv(\alpha\mK)) + \eta{}F)^2 - (k_t-1)\epsilon^2 \\
& = &\dist^2(\tilde{\y}_{t}, \conv(\alpha\mK)) + 2\eta{}F\dist(\tilde{\y}_{t}, \conv(\alpha\mK)) + \eta^2F^2 - k_t\epsilon^2 + \epsilon^2.
\end{eqnarray*}

Rearranging, summing over all $T$ iterations, and recalling that for all $t$, $\dist(\tilde{\y}_{t}, \conv(\alpha\mK)) \leq \sqrt{2}\alpha{}R$, we have that
\begin{eqnarray*}
\sum_{t=1}^{T-1}k_t &\leq &\frac{1}{\epsilon^2}\left({\dist^2(\tilde{\y}_1, \conv(\alpha\mK))-\dist^2(\tilde{\y}_T, \conv(\alpha\mK)) + (T-1)\left({2\sqrt{2}\eta\alpha{}RF + \eta^2F^2+ \epsilon^2}\right)}\right) \\
& \leq & (T-1)\left({1+  \frac{2\sqrt{2}\eta\alpha{}RF + \eta^2F^2}{\epsilon^2}}\right).
\end{eqnarray*}

\end{proof}

\subsection{Algorithm for the bandit information setting}

Our algorithm for the bandit setting follows from a very well known reduction from the bandit setting to the full information setting, also applied in the bandit algorithm of \cite{KKL}. The algorithm simply simulates the full information algorithm, Algorithm \ref{alg:newOGD}, by providing it with estimated loss/payoff vectors $\hat{\f}_1,...,\hat{\f}_T$ instead of the true vectors $\f_1,...,\f_T$ which are not available in the bandit setting. This reduction is based on the use of a \textit{Barycentric Spanner} (defined next) for the feasible set $\mK$, which is crucial for obtaining the estimates $\hat{\f}_1,...,\hat{\f}_T$. 
As standard when using this approach, we assume that the points in $\mK$ span the entire space $\reals^d$, otherwise we can reformulate the problem in a lower-dimensional space, in which this assumption holds.

\begin{definition}[Barycentric Spanner]
We say that a set of $d$ vectors $\{\q_1,...,\q_d\}\subset\reals^d$ is a \textit{Barycentric Spanner} with parameter $\beta > 0$ for a set $\mS\subset\reals^d$, denoted by \BS($\mS$), if it holds that $\{\q_1,...,\q_d\}\subset\mS$, and the matrix $\Q := \sum_{i=1}^d\q_i\q_i^{\top}$ is not singular and $\max_{i\in[d]}\Vert{\Q^{-1}\q_i}\Vert \leq \beta$.
\end{definition}
We note that while this definition is somewhat different than the classical one given in \cite{Awerbuch04}, it is in-fact equivalent to the $C$-approximate barycentric spanner considered in \cite{Awerbuch04}, with an appropriately chosen constant $C(\beta)$.

Importantly, as discussed in \cite{KKL}, the assumption on the availability of such a set \BS($\mK$) seems reasonable, since i) for many sets that correspond to the set of all possible solutions to some well-studied NP-Hard optimization problem, one can still construct in $\textrm{poly}(d)$ time a barycentric spanner with $\beta = \textrm{poly}(d)$, ii) \BS($\mK$) needs to be constructed only once and then stored in memory (overall $d$ vectors in $\reals^d$), and hence its construction can be viewed as a pre-processing step, and iii) as illustrated in \cite{KKL}, without further assumptions, the approximation oracle by itself is not sufficient to guarantee nontrivial regret bounds in the bandit setting.

\begin{algorithm}
\caption{Bandit Algorithm}
\label{alg:band}
\begin{algorithmic}[1]
\STATE input: learning rate $\eta > 0$, projection error parameter $\epsilon > 0$, $\{\q_1,...,\q_d\}$ - a \BS($\mK$) for some $\beta > 0$, exploration parameter $\gamma\in(0,1)$
\STATE instantiate Algorithm \ref{alg:newOGD} with parameters $(\eta, \epsilon)$
\FOR{$t=1\dots T$}
\STATE receive $(\s_t, \tilde{\y}_t)\in\mK\times\ball(0,\alpha{}R)$ from Algorithm \ref{alg:newOGD}
\STATE $b_t \gets \left\{ \begin{array}{ll}
         \textsc{explore} & \mbox{with prob. $\gamma$} \\ 
          \textsc{exploit}& \mbox{with prob. $1-\gamma$}\end{array} \right. $
\IF{$b_t = \textsc{explore}$}
\STATE sample $i_t\in[d]$ uniformly at random
\STATE play $\hat{\s}_t = \q_{i_t}$
\STATE receive loss/payoff $\ell_t = \q_i\cdot\f_t$
\STATE set $\hat{\f}_t \gets \frac{d\ell_t}{\gamma}\Q^{-1}\q_{i_t}$ \COMMENT{recall $\Q = \sum_{i=1}^d\q_i\q_i^{\top}$}
\ELSE
\STATE play $\hat{\s}_t = \s_t$
\STATE receive loss/payoff $\ell_t = \s_t\cdot\f_t$
\STATE set $\hat{\f}_t \gets \textbf{0}$
\ENDIF
\STATE feed $\hat{\f}_t$ to Algorithm \ref{alg:newOGD} as the loss/payoff vector for round $t$
\ENDFOR
\end{algorithmic}
\end{algorithm}

\begin{theorem}\label{thm:newBandit}
Fix $\eta > 0, \epsilon \in (0,~(\alpha+2)R] , \gamma\in(0,1)$. Suppose Algorithm \ref{alg:band} is applied for $T$ rounds and let $\{\f_t\}_{t=1}^T\subseteq\mF$ be the sequence of observed loss/payoff vectors, and let $\{\hat{\s}_t\}_{t=1}^T$ be the sequence of points played by the algorithm.
Then it holds that
\begin{eqnarray*}
\E\left[{\alpharegret}\left({\{(\hat{\s}_t,\f_t)\}_{t\in[T]}}\right)\right] \leq  \frac{\alpha^2R^2}{\eta{}T} + \frac{\eta{}d^2C^2\beta^2}{2\gamma} + 3\epsilon{}F +  \gamma{}C,
\end{eqnarray*}
and the expected number of calls to the approximation oracle of $\mK$ per iteration is upper bounded by
\begin{eqnarray*}
\E\left[{K(\eta, \epsilon, \gamma)}\right] := O\left({\left({1+  \frac{\eta\alpha\beta{}dCR + (\eta{}dC\beta)^2/\gamma}{\epsilon^2}}\right)d^2\ln\left({\frac{(\alpha+1)R}{\epsilon}}\right)}\right).
\end{eqnarray*}
In particular, setting $\eta = \frac{\alpha{}R}{\beta{}dC}T^{-2/3}$, $\epsilon = \alpha{}RT^{-1/3}$, $\gamma = T^{-1/3}$ gives
\begin{eqnarray*}
\E\left[{\alpharegret}\right] = O\left({(\alpha\beta{}dCR+\alpha{}RF + C)T^{-1/3}}\right), \quad \E[K] = O\left({d^2\ln\left({\frac{\alpha+1}{\alpha}T}\right)}\right).
\end{eqnarray*}

\end{theorem}

\begin{proof}
The proof is very similar to that of Theorem \ref{thm:newOGD} and we focus on the modifications of it required to prove Theorem \ref{thm:newBandit}.  Again, we focus on the case $\alpha \geq 1$ since the complementary case follows the same lines with the obvious modifications.
Let $\x^*\in\arg\min_{\x\in\mK}\sum_{t=1}^T\x\cdot\f_t$. Applying Lemma \ref{lem:ogdRegret} with respect to the sequence of points $\{\tilde{\y}_t\}_{t=1}^T$ and the sequence of losses $\{\hat{\f}_t\}_{t=1}^T$, we have that
\begin{eqnarray*}
\sum_{t=1}^T\tilde{\y}_t\cdot\hat{\f}_t - \alpha\cdot\sum_{t=1}^T\x^*\cdot\hat{\f}_t \leq  \frac{\alpha^2R^2}{\eta} + \frac{\eta}{2}\sum_{t=1}^T\Vert{\hat{\f}_t}\Vert^2.
\end{eqnarray*}
Taking expectation with respect to the random variables $b_1,i_1,...,b_T,i_T$ and noting that for all $t\in[T]$, both $\x^*$ and $\tilde{\y}_t$ are independent of the randomness in $\hat{\f}_t$, we have that
\begin{eqnarray*}
\E_{\{(b_t,i_t)\}_{t=1}^T}\left[{\sum_{t=1}^T\tilde{\y}_t\cdot\f_t}\right] - \alpha\cdot\sum_{t=1}^T\x^*\cdot\f_t \leq  \frac{\alpha^2R^2}{\eta} + T\frac{\eta{}d^2C^2\beta^2}{2\gamma},
\end{eqnarray*}
were we have used the observations that
\begin{eqnarray*}
\E_{b_t,i_t}[\hat{\f}_t] &=& \gamma\sum_{i=1}^d\frac{1}{d}\cdot\frac{d\q_{i}^{\top}\f_t}{\gamma}\Q^{-1}\q_{i}
= \sum_{i=1}^d\Q^{-1}\q_{i}\q_{i}^{\top}\f_t = \Q^{-1}\Q\f_t = \f_t,
 \\
\E_{b_t}[\Vert{\hat{\f}_t}\Vert^2] &=& \gamma\frac{d^2}{\gamma^2}\ell_t^2\Vert{\Q^{-1}\q_{i_t}}\Vert^2 + (1-\gamma)0\leq \frac{(dC\beta)^2}{\gamma}.
\end{eqnarray*}
As in the proof of Theorem \ref{thm:newOGD}, 
for every iteration $t\geq 1$, let us denote $\p_{t+1} = \sum_{i=1}^{N}a_i\v_i$, $\bar{\s}_t = \sum_{i=1}^{N}a_i\s_i$, where $(a_1,...,a_{N})$, $\{(\v_1,\s_1),...,(\v_{N_t},\s_{N})\}$ are the outputs of the call to Algorithm \ref{alg:approxProj} on that iteration. Also define $\p_1 = \tilde{\y}_1$. Again, by the guarantee of Lemma \ref{lem:approxProj}, we have that
\begin{eqnarray*}
\E_{\{(b_t,i_t)\}_{t=1}^T}\left[{\sum_{t=1}^T\p_t\cdot\f_t}\right] - \alpha\cdot\sum_{t=1}^T\x^*\cdot\f_t \leq  \frac{\alpha^2R^2}{\eta} + T\frac{\eta{}d^2C^2\beta^2}{2\gamma}+ 3T\epsilon{}F.
\end{eqnarray*}
Since $\p_t$ is dominated by $\bar{\s}_t = \E[\s_t]$ for all $t\in[T]$, we have that
\begin{eqnarray*}
\E_{\{(b_t,i_t,\s_t)\}_{t=1}^T}\left[{\sum_{t=1}^T\s_t\cdot\f_t}\right] - \alpha\cdot\sum_{t=1}^T\x^*\cdot\f_t \leq  \frac{\alpha^2R^2}{\eta} + T\frac{\eta{}d^2C^2\beta^2}{2\gamma}+ 3T\epsilon{}F.
\end{eqnarray*}
Finally, since 
\begin{eqnarray*}
\forall t\in[T]: \quad \E_{b_t}[\hat{\s}_t\cdot\f_t] = (1-\gamma)\s_t\cdot\f_t+ \gamma\q_{i_t}\cdot\f_t 
\left\{ \begin{array}{ll}
         \leq \s_t\cdot\f_t + \gamma{}C & \mbox{if $\alpha \geq 1$} \\ 
          \geq \s_t\cdot\f_t - \gamma{}C& \mbox{if $\alpha < 1$}\end{array} \right.,
\end{eqnarray*}
we have that
\begin{eqnarray*}
\E\left[{\sum_{t=1}^T\hat{\s}_t\cdot\f_t}\right] - \alpha\cdot\sum_{t=1}^T\x^*\cdot\f_t \leq  \frac{\alpha^2R^2}{\eta} + T\frac{\eta{}d^2C^2\beta^2}{2\gamma}+ 3T\epsilon{}F +  T\gamma{}C,
\end{eqnarray*}
as required.

We now turn to upper bound the overall number of queries to the approximation oracle of $\mK$. Note that we require to compute a new approximated projection only after rounds for which it holds that $b_t = \textsc{EXPLORE}$, since otherwise it holds that $\hat{\f}_t = \textbf{0}$, and there is no update to the iterates maintained by Algorithm \ref{alg:newOGD}. For any $t\in[T]$ we define the indicator variable:
\begin{eqnarray*}
I_t \gets \left\{ \begin{array}{ll}
         \textsc{1} & \mbox{if $b_t = \textsc{EXPLORE}$} ;\\ 
          \textsc{0}& \mbox{if $b_t = \textsc{EXPLOIT}$.}\end{array} \right.
\end{eqnarray*}

Define $\hat{F} := \frac{dC\beta}{\gamma}$, and observe that for all $t\in[T]$ it holds that 
\begin{eqnarray*}
\Vert{\hat{\f}_t}\Vert \leq \left\|{\frac{d\Q^{-1}\q_{i_t}\ell_t}{\gamma}}\right\| \leq \frac{d}{\gamma}\vert{\q_{i_t}\cdot\f_t}\vert\cdot\Vert{\Q^{-1}\q_{i_t}}\Vert
 \leq \frac{d}{\gamma}C\beta = \hat{F}.
\end{eqnarray*}

Now, we continue to bound the number of calls to Algorithm \ref{alg:approxProj}, very similarly to the analysis in the proof of Theorem \ref{thm:newOGD}.

Let $k_t$ be the number of iterations it took Algorithm \ref{alg:approxProj} to terminate when invoked on iteration $t$ of Algorithm \ref{alg:newOGD} (w.l.o.g. this happens when Algorithm \ref{alg:band} sends the feedback $\hat{\f}_t$ to Algorithm \ref{alg:newOGD}), and note that $\E[K(\eta, \epsilon,\gamma)] = \frac{1}{T}\E\left[{\sum_{t=1}^{T-1}k_t}\right]\cdot{}O\left({d^2\ln\left({\frac{(\alpha+1)R}{\epsilon}}\right)}\right)$. Note that for all $t\geq 1$, $\y_{t+1} = \tilde{\y}_t - I_t\eta\hat{\f}_t$. Thus, by Lemma \ref{lem:approxProj}, it follows that on any iteration $t$, 
\begin{eqnarray*}
\dist^2(\tilde{\y}_{t+1}, \conv(\alpha\mK)) &\leq& \dist^2(\y_{t+1}, \conv(\alpha\mK)) - (k_t-1)\epsilon^2 \\
 &=& \dist^2(\tilde{\y}_{t}-I_t\eta\hat{\f}_{t}, \conv(\alpha\mK)) - (k_t-1)\epsilon^2 \\
&\leq& (\dist(\tilde{\y}_{t}, \conv(\alpha\mK)) + I_t\eta{}\hat{F})^2 - (k_t-1)\epsilon^2 \\
& = &\dist^2(\tilde{\y}_{t}, \conv(\alpha\mK)) + 2I_t\eta{}\hat{F}\dist(\tilde{\y}_{t}, \conv(\alpha\mK)) + I_t\eta^2\hat{F}^2 - k_t\epsilon^2 + \epsilon^2.
\end{eqnarray*}

Rearranging, summing over all iterations $1...T-1$, and recalling that for all $t$, $\dist(\tilde{\y}_{t-1}, \conv(\alpha\mK)) \leq \sqrt{2}\alpha{}R$, we have that
\begin{align*}
\sum_{t=1}^{T-1}k_t \leq \frac{1}{\epsilon^2}\left({\dist^2(\tilde{\y}_1, \conv(\alpha\mK))-\dist^2(\tilde{\y}_T, \conv(\alpha\mK)) + 2\sqrt{2}\sum_{t=1}^{T-1}I_t\eta\alpha\hat{F}R + \sum_{t=1}^{T-1}I_t\eta^2\hat{F}^2+ (T-1)\epsilon^2}\right).
\end{align*}

Taking expectation with respect to the random variables $I_1,...,I_{T-1}$ we have that

\begin{eqnarray*}
\E\left[{\sum_{t=1}^{T-1}k_t}\right] &\leq & (T-1)\left({1+  \frac{2\sqrt{2}\gamma\eta\alpha\hat{F}R +\gamma\eta^2\hat{F}^2}{\epsilon^2}}\right) \\
&=& (T-1)\left({1+  \frac{2\sqrt{2}\eta\alpha\beta{}dCR + (\eta{}dC\beta)^2/\gamma}{\epsilon^2}}\right),
\end{eqnarray*}
as required.

\end{proof}

\section{Open Problems}
It remains open to find algorithms that guarantee both $\tilde{O}(T^{-1/2})$ $\alpha$-regret and $\textrm{poly}(\log{T})$ oracle complexity per iteration (at least on average), for both the full information and bandit settings.

\bibliographystyle{plain}
\bibliography{bib}

\appendix

\section{Proofs Omitted from Section \ref{sec:KKL}}

For clarity, below we present the algorithm from \cite{KKL}, implied by Lemma \ref{lem:KKL:1}, for computing approximated infeasible projections onto $\conv(\alpha\mK)$  in full-detail, see Algorithm \ref{alg:fw}.

\subsection{Proof of Lemma \ref{lem:KKL:1}}
\begin{proof}
To prove the first part of the lemma, suppose that $\x$ satisfies that $(\x-\y_{t+1})\cdot(\x-\v') \leq \epsilon$, where $(\v',\s') \gets \oraclekext(\x-\y_{t+1})$.
Fix $\z\in\conv(\alpha\mK)$. It holds that
\begin{eqnarray*}
\Vert{\y_{t+1}-\z}\Vert^2 &=& \Vert{(\y_{t+1}-\x) + (\x - \z)}\Vert^2 = \Vert{\y_{t+1}-\x}\Vert^2 + \Vert{\x-\z}\Vert^2 -2(\x-\y_{t+1})\cdot(\x-\z) \\
&\geq &\Vert{\x-\z}\Vert^2 - 2(\x-\y_{t+1})\cdot(\x-\z) \\
&\geq & \Vert{\x-\z}\Vert^2 - 2(\x-\y_{t+1})\cdot(\x-\v')
\geq \Vert{\x-\z}\Vert^2 - 2\epsilon,
\end{eqnarray*}
where the second inequality holds since $\v'$ is the output of the extended approximation oracle with respect to the input $(\x-\y_{t+1})$.

For the second part of the lemma, we observe that if $(\x-\y_{t+1})\cdot(\x-\v') > \epsilon$, then
\begin{eqnarray*}
\Vert{\x' - \y_{t+1}}\Vert^2 &=& \Vert{\x - \y_{t+1} + \lambda(\v'-\x)}\Vert^2 \\
&=& \Vert{\x-\y_{t+1}}\Vert^2 - 2\lambda(\x-\y_{t+1})\cdot(\x-\v') + \lambda^2\Vert{\v'-\x}\Vert^2 \\
& \leq & \Vert{\x-\y_{t+1}}\Vert^2 - 2\lambda\epsilon+ 2\lambda^2(\Vert{\v'}\Vert^2 + \Vert{\x}\Vert^2)\\
& \leq & \Vert{\x-\y_{t+1}}\Vert^2 - 2\lambda\epsilon+ 4\lambda^2(\alpha+2)^2R^2,
\end{eqnarray*}
where the first inequality follows since $(\x-\y_{t+1})\cdot(\x-\v') > \epsilon$ and using the triangle inequality with $(a+b)^2 \leq 2(a^2+b^2)$, and the second inequality follows by the assumption on $\x$ and since $\v'$ is the output of the extended approximation oracle.
Thus, we can see that setting $\lambda = \frac{\epsilon}{3(\alpha+2)^2R^2} \in(0,1]$, gives the requested result.

Finally, since $\x$ and $\v'$ are dominated by $\s$ and $\s'$ for any $\f\in\mF$, respectively, we have that $\x' = (1-\lambda)\x+\lambda\v'$ is dominated by $(1-\lambda)\s+\lambda\s'$ for any $\f\in\mF$.
\end{proof}

\subsection{Proof of Lemma \ref{lem:KKL:2}}
\begin{proof}
We begin by proving the regret bound. Since each $\x_{t+1}$ is an approximated projection of $\y_{t+1}$ in the sense that
\begin{eqnarray*}
\forall \z\in\conv(\alpha\mK): \quad \Vert{\z-\x_{t+1}}\Vert^2 \leq \Vert{\z-\y_{t+1}}\Vert^2 + 2\epsilon,
\end{eqnarray*}
it is immediate to see from the proof of Lemma \ref{lem:ogdRegret}, that incorporating this approximation error into the regret bound, and bounding $\Vert{\f_t}\Vert \leq F$ for all $t$, results in the regret bound:
\begin{eqnarray*}
\frac{1}{T}\sum_{t=1}^T\x_t \cdot \f_t - \min_{\x\in\conv(\alpha\mK)}\frac{1}{T}\sum_{t=1}^T\x\cdot \f_t &=&
\frac{1}{T}\sum_{t=1}^T\x_t \cdot \f_t - \alpha\min_{\x\in\mK}\frac{1}{T}\sum_{t=1}^T\x\cdot \f_t \\
&\leq& \frac{\alpha^2R^2}{T\eta} + \frac{\eta{}F^2}{2}  + \frac{\epsilon}{\eta}.
\end{eqnarray*}
The regret bound now follows by recalling that for all $t$ and all $\f\in\mF$: $\E[\s_t\cdot\f_t] = \bar{\s}_t\cdot\f_t \leq \x_t\cdot\f_t$, and taking expectation with respect to the randomness in $\s_t$.

To bound the number of calls to the approximation oracle per some iteration $t$, note that $\Vert{\x_t - \y_{t+1}}\Vert^2 \leq \eta^2F^2$. Thus, if we initialize the projection algorithm, described in Lemma  \ref{lem:KKL:1}, with the point $\x_t$, and we recall that each iteration of the algorithm reduces the potential $\Vert{\x - \y_{t+1}}\Vert^2$ by $\Omega(\epsilon^2)$, where $\x$ is the current iterate, then we have that at most $O(\eta^2F^2/\epsilon^2)$ iterations are required for the algorithm to terminate.
\end{proof}




\begin{algorithm}
\caption{Frank-Wolfe for Approximated (infeasible) Projection onto $\conv(\alpha\mK)$}
\label{alg:fw}
\begin{algorithmic}[1]
\STATE input: point to project $\y\in\reals^d$, error tolerance $\epsilon \in (0,~3(\alpha+2)^2R^2)$
\STATE output: $(\x,\bar{\s})\in\reals^d\times\conv(\mK)$ such that $\x$ is an $\epsilon$-approximated infeasible projection of $\y$ dominated by $\bar{\s}$ for any $\f\in\mF$
\STATE let $(\x_1,\bar{\s}_1)\in\reals^n\times\conv(\mK)$ such that $\x_1$ is dominated by $\bar{\s}_1$ for any $\f\in\mF$.
\STATE $\lambda \gets \epsilon/(3(\alpha+2)^2R^2)$
\FOR{$i= 1...$}
\STATE $(\v_i, \s_i) \gets \oraclekext(\x_i-\y)$
\IF{$(\x_i-\y)\cdot(\x_i-\v_i) \leq \epsilon$}
\RETURN $(\x_i,\bar{\s}_i)$
\ENDIF
\STATE $\x_{i+1} \gets \x_i + \lambda(\v_i - \x_i)$
\STATE $\bar{\s}_{i+1} \gets \bar{\s}_i + \lambda(\s_i-\bar{\s}_i)$
\ENDFOR
\end{algorithmic}
\end{algorithm}




\end{document}